\documentclass[12pt]{amsart}
\usepackage{geometry}                
\usepackage{graphicx}
\usepackage{amssymb}
\usepackage{epstopdf}

\usepackage[round]{natbib}
\usepackage{amsmath}
\usepackage{amsthm}

\newtheorem{thm}{Theorem}
\newtheorem{cor}{Corollary}
\newtheorem{lem}{Lemma}

\newtheorem{asm}{Assumption}

\begin{document}

\begin{titlepage}

\title{Stochastic Matrix Factorization}
\author{Christopher P. Adams, \\
Federal Trade Commission}
\thanks{The views expressed in this article are those of the author
and do not necessarily reflect those of the Federal Trade
Commission.  I'm grateful for continued discussions about this problem with Devesh Raval and Nathan Wilson.  Thanks to Matthew Chesnes for his assistance with the dissertation abstract data. All errors are my own.}

\email{cadams@ftc.gov}
\maketitle

\date{} 

\begin{abstract}
\footnotesize{This paper considers a restriction to non-negative matrix factorization in which at least one matrix factor is stochastic.  That is, the elements of the matrix factors are non-negative and the columns of one matrix factor sum to 1.  This restriction includes topic models, a popular method for analyzing unstructured data.  It also includes a method for storing and finding pictures.  The paper presents necessary and sufficient conditions on the observed data such that the factorization is unique.  In addition, the paper characterizes natural bounds on the parameters for any observed data and presents a consistent least squares estimator.  The results are illustrated using a topic model analysis of PhD abstracts in economics and the problem of storing and retrieving a set of pictures of faces.}
\end{abstract}

\end{titlepage}

\section{Introduction}

Matrix factorization is an important tool across applied statistics including psychometrics, econometrics, biostatistics and machine learning.  Despite its popularity, matrix factorization suffers from two fundamental problems.  First, in general the matrix factorization does not have a unique solution.  There may be many unrelated matrices consistent with the observed data.  Second, the number of parameters to estimate may be larger than the number of observations and increasing at a faster rate.  In the language of statistics, estimators based on matrix factorization may neither be identified nor consistently estimated.  In practice, results based on matrix factorization may be neither reliable nor robust.

This paper considers a subset of matrix factorization problems where the matrix factors are assumed to have non-negative elements and for at least one matrix factor, its columns sum to one.  For this subset of problems, the paper presents necessary and sufficient conditions on the observed data for the factorization to be unique.  The paper also characterizes ``natural'' bounds of the identified set for any observed data set.  In addition, the paper shows that least squares estimators can be consistently estimated as the number of observations gets large.  The theoretical results are illustrated with a topic model analysis of Ph.D. dissertation abstracts in economics and the problem of converting pictures to a smaller data set for storage and retrieval.

In 1933, the statistician and economist, Harold Hotelling, poured cold water on the idea of using matrix factorization to solve a problem plaguing educational psychology.  The problem was that the researchers had access to data with a large number of psychometric tests given to a small number of individuals.  Hotelling pointed out that the proposed solution to this ``wide data'' problem was not unique.  In fact, there are $R^2$ parameters not tied down by the data, where $R$ is the number of ``hidden factors'' or ``components'' in Hotelling's preferred terminology.  Hotelling proposed ``principal components analysis'' as an ``orderly'' method of choosing from among the infinite solutions.

Non-uniqueness may lead to a myriad of problems.  First, it may make estimation difficult as the estimator bounces between the many solutions.  The fastest algorithms available will not help when there is no unique solution to the problem.  The problem is ill-posed.  Second, the solution found by the algorithm may not be robust.  Any small change in the initial conditions may lead to vastly different results.  For example, non-uniqueness in topic models leads to different topics and different distributions of words over those topics.  Third, these models are often used for prediction and so non-uniqueness may lead to wildly inaccurate results.  If the technique is used in picture storage and retrieval is not unique, then there will be problems in recreating or finding certain pictures.  Various techniques, such as regularization, ``tuning'' of parameters or parametric assumptions may be used to solve the non-uniqueness problem. However, if such techniques are used it becomes difficult to determine how much of the results are being driven by the data and how much are being driven by arbitrary choices of the researcher.  Understanding when an algorithm will or will not produce a unique solution provides researchers and practitioners with greater comfort in relying on the results.  

A common solution to the uniqueness problem in matrix factorization problems is to place some restriction on the assumed data generating process.  \cite{BaiNg:2013} shows that by placing a number of restrictions on the matrix factors including orthonormality and various sparsity constraints, it is possible to have a unique solution that replicates principal components.  Singular value decomposition is an example of this approach.  While the approach has the advantage of providing estimable and robust algorithms, it has the disadvantage that the assumptions are strong and ad-hoc.  In particular, the data generating process may not actually satisfy the conditions assumed.  Finding a robust solution to the wrong problem is not valuable.

An alternative solution is based on a proposal illustrated by \cite{Lee:1999}.  The authors note that some important problems lead naturally to restrictions on the matrix factors.  The factorization of a data set of pictures can naturally be constrained to have non-negative matrix factors.  \cite{Huang:2013} present necessary conditions on the observed data for a non-negative matrix factorization to have a unique solution.  The authors also present sufficient conditions on the observed data for uniqueness.\footnote{\cite{Pan:2016} also presents a characterization of the non-uniqueness for NMF problems.}  Borrowing results from the econometrics literature on mixture models, this paper shows that the \cite{Huang:2013} necessary conditions are sufficient if the problem is further restricted such that at least one of the matrix factors is stochastic.  \cite{Adams:2016} shows that requiring that the columns of one of the matrix factors to add to one reduces the number of free parameters from $R^2$ to $R(R-1)$.  If the data also has what the econometrics literature has called ``tail restrictions,'' or sparsity conditions to use the computer science and engineering terminology, then the factorization is unique.  More generally, the factorization is not unique even after requiring that the matrix factors be non-negative and that at least one matrix factor be stochastic.  The paper shows that that identified set or ``natural bounds'' are tighter when the data is more informative as to which hidden factor is likely.

While requiring an SMF is an important restriction, the practical importance is less clear.  For example, both problems analyzed in the original NMF paper \citep{Lee:1999} can be naturally modeled as SMF problems.  As we see below, topic models are naturally SMF problems, in fact for topic models both matrix factors are stochastic.  Similarly, the analysis of faces described in \cite{Lee:1999} can be modeled as a SMF problem.  The image is assumed to be a convex combination of ``base images.''  The weighting matrix is stochastic.

\cite{Lee:1999} considers the idea of taking a set of pictures of faces and using matrix factorization to reduce the size of the data set and increase the speed to find pictures in the data set.  The original data has over two thousand photos and each photo is represented by 81 pixels (in the version analyzed here), giving almost two hundred thousand values to be stored in memory.  By factoring the data into 10 base images, we can reduce the storage problem by almost 90\%.  Additionally, we can speed up the retrieval problem by reducing the number of points to be checked in each picture from 81 to 10.  The trade off is that the matrix factorization is an approximation of the original data.  The smaller number of base images, the poorer the approximation but the smaller the data and the quicker the retrieval.  The more base images, the better the approximation, but the bigger the data set and the slower the retrieval.    

A standard topic model can be thought of as a non-negative matrix factorization problem where both matrix factors are stochastic.  One matrix factor represents probability distributions over hidden topics conditional on the document and the other represents probability distributions over terms (words) conditional on the topic.  The matrix factor representing the distribution over terms, is likely to be relatively sparse while the matrix factor representing the distribution over topics may not be sparse.  \cite{Huang:2013} shows that if there exist ``anchor'' words and anchor documents, then the factorization is unique.  That is, for each topic there is a term that only occurs for that topic, similarly there are a number of documents that have only one topic.  \cite{Huang:2013} point out that having anchor words and documents is sufficient but not necessary.  The necessary condition is something of the photo negative of the anchor words and documents requirement.  

The necessary condition for uniqueness states that the set of occurring terms for each topic are not subsets of each other.  That is, there must be some terms that have ``anchor'' like properties, but they don't quite have to be unique to that topic.  The difference between the necessary condition of \cite{Huang:2013} and the anchor condition is more obvious in relations to the distribution over topics.  The necessary condition states that there exist enough documents in which one of each of the topics does not occur.  That is, there has to be enough variation in the documents in regards to which topics they cover.  In particular, there must be documents that do not cover certain topics.  The anchor property states that there must be a certain number of documents with only one topic.  This property is obviously a lot stricter than requiring that a certain set of documents \emph{not} include one topic.  Moreover, it is unlikely to hold in documents longer than say tweets.  The result presented below shows that the necessary condition of \cite{Huang:2013} is also sufficient for stochastic matrix factorization problems such as topic models.

In order to check for uniqueness of our topic model estimator we simply check the \cite{Huang:2013} conditions on the two estimated matrix factors.  If the conditions hold, then it is likely that our estimated result is unique.  What if the conditions do not hold?  In this case, the non-negative conditions place ``bounds'' on the estimated parameters.  \citep{Tian:2000} calls these ``natural bounds'' as they are derived from natural assumptions that probabilities are non-negative and add to 1.  Manski calls them ``worst-case'' bounds \citep{Ho:2015}.  \cite{Henry:2014} and \cite{Adams:2016} characterize these bounds.  Although the bounds get difficult to characterize when there are more than two hidden factors, it is always straightforward to characterize a subset of the bounds that sits on the coordinate axes.  This subset is characterized by likelihood ratios.  If a term is relatively more likely to occur for one topic versus an alternative topic, then the bounds will be tighter.  Similarly, if a document is more likely to incorporate a particular topic relative to a different document, then the bounds will be tighter.  If the width of the bounds along the coordinate axes is small in every dimension, then the the identified set is small.  

The second fundamental problem with matrix factorization based estimation is that there may exist more parameters than observations, with the number of parameters growing faster than the number of observations.  Consistent estimation is important for insuring that the solution to the algorithm represents something true about the world.  If this is not true, then results may vary substantially with variation in the data and predicted values may not be close to their true values.  One solution to the ``too many parameters'' problem is to assume that certain sets of parameters are drawn from a distribution \citep{BLP:1995}.  This reduces the parameter space from the full set of parameters to just the parameters describing the distribution.  This is the solution used in the latent Dirchlet allocation approach to topic modeling \citep{Blei:2009,Blei:2012}.  The advantage of this approach is that it may be possible to have a consistent estimator, the disadvantage is that it relies on strong parametric assumptions that may not hold in the underling document generation process.  For example, the standard LDA model assumes that the distributions across topics is independent \citep{Blei:2009}.  

The approach used here is similar to the approach suggested by \cite{BaiNg:2013}.  The idea is to ``concentrate out'' one of the matrix factors, so the estimation procedure only estimates the parameters of one of the two matrix factors.  If for example there are 50 terms, 10 topics and 10,000 documents, the procedure reduces the number of parameters to be estimated from 90,490 to 490 (noting that the adding up assumption).  Moreover the proposed procedure is ``non-parametric,'' making it a much more robust algorithm.  The algorithm it is not reliant on assumptions that are unlikely to hold in the actual underling document generating process.  The procedure proposed below uses least squares to estimate the distribution over terms for each topic.  The distributions over topics for each topic can be determined by multiplying the observed data (the document-term matrix) and multiplying that by the generalized inverse of the estimated matrix factor.  The paper shows that as the number of documents gets large then the model estimates converge to the true values in probability.

Matrix factorization is used in a vast array of problems.  The common attribute of these problems is correlation in the data.  In the face problem, the ``variables'' are the different pixel positions and there is a high degree of correlation among pixels that are close to each other but also other pixels that may be associated with a different part of the face based on the lighting and complexion of the face.  In the topic model, the words are often used in conjunction with each other.  Matrix factorization can be used to characterize a document corpus or a set of picture if the correlation in the data can be characterized as a mixture across a small set of ``hidden factors.''

The paper continues as follows.  Section 2 presents the stochastic matrix factorization model.  Section 3 presents the main identification results.  Section 4 presents the estimation algorithm used for the picture reduction problem and the topic model analysis.  Section 5 presents the empirical analysis of the picture reduction problem.  Section 6 presents the empirical analysis of the PhD dissertation abstracts.  Section 6 concludes.

\section{Stochastic Matrix Factorization}

Consider a ``data matrix'' such as document-term matrix $\mathbf{X}$, where it is $N \times M$.  The objective is to factor this matrix giving two matrix factors, a matrix $\mathbf{W}$ which is $N \times R$ and a matrix $\mathbf{H}$ which is $R \times M$.  Note that the results presented here also hold, with some minor tweaks, for mixture models \citep{Adams:2016}.

\begin{equation}
\mathbf{X} = \mathbf{W} \mathbf{H}
\label{eq:XWH}
\end{equation}

Note that this factorization holds for any $R > 0$.  In this paper we will make no claims as to what $R$ is, although it should be noted that the data requirements for identification of models with larger $R$ may be higher.

The identification problem is straightforward to understand.  Consider a matrix $\mathbf{A}$ which is $R \times R$ and full-rank.  Equation (\ref{eq:XWH}) is consistent with any such matrix.

\begin{equation}
\mathbf{X} = \mathbf{W} \mathbf{A} \mathbf{A}^{-1} \mathbf{H}
\label{eq:XWAH}
\end{equation}

That is, the observed data and the assumptions of the data generating process that lead naturally to matrix factorization, provide no restrictions on an $R \times R$ matrix.

As discussed above we will make two ``natural'' assumptions that will be shown to restrict the matrix $\mathbf{A}$ and in some circumstances the only $\mathbf{A}$ that satisfies Equation (\ref{eq:XWAH}) is the identity matrix (up to rearranging the columns).  In this case the matrix factorization is unique (up to rearranging columns).

\begin{asm} (Non-Negativity) For all $\mathbf{W}$, $\mathbf{H}$ and $\mathbf{A}$ satisfying Equation (\ref{eq:XWAH}) assume that
\begin{enumerate}
\item $\mathbf{W} \mathbf{A} \ge 0$
\item $\mathbf{A}^{-1} \mathbf{H} \ge 0$
\end{enumerate}
where the notation $\mathbf{B} \ge 0$ refers to each element of the matrix being greater or equal to zero.
\label{asm:nonneg}
\end{asm}

The first assumption is the standard non-negativity assumption.  It states that each and every cell of the matrix factors must either be zero or positive.  This is a natural assumption if the matrix factors are to represent probability distributions as they do in topic models.

\begin{asm} (Adding up)  Given $\mathbf{H}$ and $\mathbf{A}$ satisfying Equation (\ref{eq:XWAH}) assume that either 
\begin{enumerate}
\item $\sum_{m =1}^M \tilde{\mathbf{H}}_{rm} = 1 \mbox{ for all } r \in \{1,...,R\}$, or
\item $\sum_{r =1}^R \tilde{\mathbf{H}}_{rm} = 1 \mbox{ for all } m \in \{1,...,M\}$
\end{enumerate}
where $\tilde{\mathbf{H}} = \mathbf{A}^{-1} \mathbf{H} $.
\label{asm:addup}
\end{asm}

The second assumption is that one of the matrix factors has rows or columns that sum to one.  In the topic model problem it is assumed that there is a probability distribution over terms for each topic, that is the rows of the $\mathbf{H}$ matrix sum to 1.  In the face problem the observed picture is assumed to be made up of a convex combinations of ``base images.''  Note that assuming both matrix factors are stochastic provides no additional identifying power.

Given these assumptions we can restrict the set of matrices, $\mathbf{A}$, that can satisfy Equation (\ref{eq:XWAH}).

\begin{lem} Given Assumption \ref{asm:addup} and assuming $\mathbf{H}$ is rank $R$, the matrix $\mathbf{A}$ that satisfies Equation (\ref{eq:XWAH}) is such that $\mathbf{A} \mathbf{1}_R = \mathbf{1}_R$, where $\mathbf{1}_R$ is a $R \times 1$ column vector of 1's.
\label{lem:Asum1}
\end{lem}

\begin{proof} Step 1.\\

First consider (1) of Assumption \ref{asm:addup}  and if $\mathbf{A}$ is the identity matrix, we have
\begin{equation}
\mathbf{H} \mathbf{1}_M = \mathbf{1}_R
\end{equation}
Also by Assumption \ref{asm:addup} for any $\mathbf{A}$ we have
\begin{equation}
\mathbf{A}^{-1} \mathbf{H} \mathbf{1}_M = \mathbf{1}_R 
\end{equation}
By substitution
\begin{equation}
\mathbf{A}^{-1} \mathbf{1}_R = \mathbf{1}_R 
\end{equation}
As $\mathbf{A}$ is full-rank we have the result.

Step 2.\\
Alternatively, consider (2) of Assumption \ref{asm:addup} and let $\mathbf{A}$ be the identity matrix
\begin{equation}
\mathbf{1}^T_R \mathbf{H}  = \mathbf{1}^T_M
\end{equation}
If $\mathbf{H}$ is rank $R$
\begin{equation}
\mathbf{1}^T_R  = \mathbf{1}^T_M \mathbf{H}^{+}
\end{equation}
where $\mathbf{H}^{+}$ refers to the generalized inverse.  Also, for any $\mathbf{A}$
\begin{equation}
\mathbf{1}^T_R \mathbf{A}^{-1}  = \mathbf{1}^T_M \mathbf{H}^{+}
\end{equation}
So
\begin{equation}
\mathbf{1}^T_R \mathbf{A}^{-1}  = \mathbf{1}^T_R
\end{equation}
Rearranging, we have the result.
\end{proof}

Lemma \ref{lem:Asum1} states that there are $R(R-1)$ free parameters in the stochastic matrix factorization problem rather than the $R^2$ free parameters in the original matrix factorization problem.  Below the paper shows that the implications of this lemma are actually larger than the reduction in free parameters.  Given this lemma, the next section shows that the necessary condition of \cite{Huang:2013} is also sufficient.

Note that the lemma is written assuming that it is the $\mathbf{H}$ matrix that is stochastic, but the result holds for either of the matrix factors.

\section{Identification}

To illustrate the identification or non-uniqueness problem consider a simple case where there are two topics ($R = 2$).  By Lemma \ref{lem:Asum1}, we can write the matrix $\mathbf{A}$ with two parameters.

\begin{equation}
\mathbf{A} = \left [ \begin{array}{cc}
                              1 - a  &  a \\
                              b       & 1 - b
                              \end{array} \right ]
\end{equation}

The inverse is

\begin{equation}
\mathbf{A}^{-1} = \left [ \begin{array}{cc}
                              \frac{1 - b}{1 - a - b}  &  -\frac{a}{1 - a - b} \\
                              -\frac{b}{1 - a - b}       & \frac{1 - a}{1 - a - b}
                              \end{array} \right ]
\end{equation}

Note that while the matrix factors must be non-negative, the is no requirement that the matrix $\mathbf{A}$ be non-negative.

Assumption \ref{asm:nonneg} places constraints on the two parameters $a$ and $b$.  In particular, the following inequalities must hold.

Consider the case where $1 - a - b > 0$, then for all $i \in \{1,...,N\}$ and $j \in \{1,...,m\}$,

\begin{enumerate}
\item $\mathbf{W}_{i1} (1 - a) + \mathbf{W}_{i2} b \ge 0$
\item $\mathbf{W}_{i1} a + \mathbf{W}_{i2} (1 - b) \ge 0$
\item $\mathbf{H}_{1j} (1 - b) - \mathbf{H}_{2j} a \ge 0$
\item $-\mathbf{H}_{1j} b + \mathbf{H}_{2j} (1 - a) \ge 0$
\end{enumerate}

If $b = 0$, then from inequalities, $a \in [-a_L, a_H]$, where
\begin{equation}
a_L = \min_{i \in \{1,...,N\}} \frac{\mathbf{W}_{i2}}{\mathbf{W}_{i1}}
\end{equation}
and
\begin{equation}
a_H = \min_{j \in \{1,...,M\}} \frac{\mathbf{H}_{1j}}{\mathbf{H}_{2j}}
\end{equation}
Similarly, if $a = 0$, then $b \in [-b_L, b_H]$ where
\begin{equation}
b_L = \min_{i \in \{1,...,N\}} \frac{\mathbf{W}_{i1}}{\mathbf{W}_{i2}}
\end{equation}
and
\begin{equation}
b_H = \min_{j \in \{1,...,M\}} \frac{\mathbf{H}_{2j}}{\mathbf{H}_{1j}}
\end{equation}
Note that these conditions are based on likelihood ratios.  If we observe a word that is much more likely to have come from Topic 1 than from Topic 2, then the bounds are tighter.  Similarly if we observe a document that is much more likely to be about Topic 2 rather than Topic 1, the bounds are tighter.  Figure \ref{fig:2dproof} illustrates the identified set.  The set $\mathcal{A}$ is the interior of the four inequalities.  

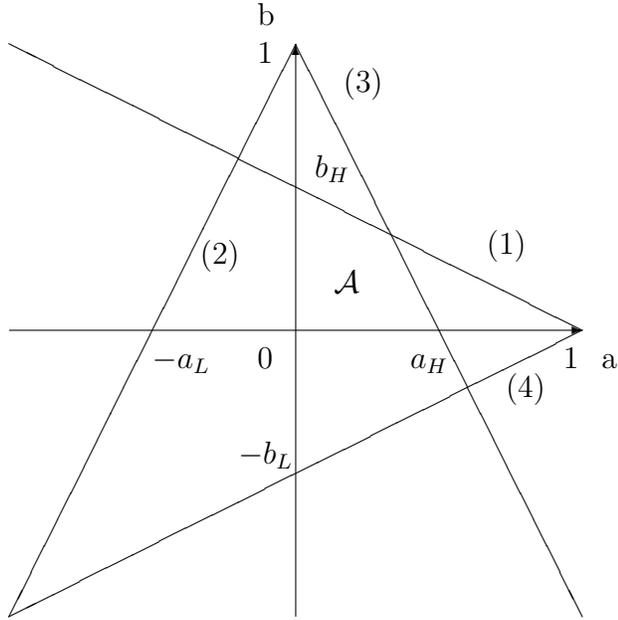
\begin{figure}[h!]
\begin{center}
\setlength{\unitlength}{.5in}
\begin{picture}(0,7)(-3,-6)
\put(-6,-3){\vector(1,0){6}}
\put(-3,-6){\vector(0,1){6}}
\thinlines
\put(0,-3){\line(-2,-1){6}}
\put(-3,0){\line(-1,-2){3}}
\put(-3,0){\line(1,-2){3}}
\put(0,-3){\line(-2,1){6}}
\put(-3.4,0.2){b}
\put(0.2,-3.4){a}
\put(-3.4,-3.4){$0$}
\put(-0.8,-3.7){(4)}
\put(-4,-2.3){(2)}
\put(-1,-2.2){(1)}
\put(-2.5,-0.5){(3)}
\put(-1.8,-3.4){$a_H$}
\put(-2.8,-1.4){$b_H$}
\put(-4.5,-3.4){$-a_L$}
\put(-3.6,-4.4){$-b_L$}
\put(-3.4,-0.2){$1$}
\put(-0.2,-3.4){$1$}
\put(-2.6,-2.6){$\mathcal{A}$}
\end{picture}
\caption{The set $\mathcal{A}$ in the two factor case.\label{fig:2dproof}}
\end{center}
\end{figure}

We also see the conditions sufficient for uniqueness.  We need four conditions to hold for $a = b = 0$.  First, there exists a word that doesn't occur for Topic 1 but does occur for Topic 2.  Second, there is a word that is the reverse.  Third there is a document that does not contain Topic 1 but does contain Topic 2.  Finally, there is a document that is the reverse.  If these four conditions hold in the observed data, then the unconstrained parameters must be zero, and the matrix $\mathbf{A}$ must be the identity matrix (up to rearranging the columns).  Note that for the two topic case these identification conditions are the ``anchor'' conditions of \cite{Arora:2012}, but for more general cases they correspond to the necessary condition presented in \cite{Huang:2013}.  This suggests that the necessary condition of \cite{Huang:2013} is also sufficient for stochastic matrix factorization.  The next result says exactly that.

\begin{thm} Let Assumptions \ref{asm:nonneg} and \ref{asm:addup} hold and define
\begin{equation}
\begin{array}{cc}
\mathcal{I}_r = \{i \in \{1,...,N\} | \mathbf{W}_{ir} \neq 0\}\\
\mathcal{J}_r = \{j \in \{1,...,M\} | \mathbf{H}_{rj} \neq 0\}
\end{array}
\end{equation}
and 
\begin{equation}
\mathbf{X} = \mathbf{W} \mathbf{A} \mathbf{A}^{-1} \mathbf{H}
\end{equation}
then $\mathbf{A} = \mathbf{I}$ (up to rearranging columns) if and only if there do not exist $r_1,r_2 \in \{1,...,R\}$, $r_1 \neq r_2$ such that $\mathcal{I}_{r_1} \subseteq \mathcal{I}_{r_2}$ or $\mathcal{J}_{r_1} \subseteq \mathcal{J}_{r_2}$.
\label{thm:unique}
\end{thm}

\begin{proof}  See \cite{Adams:2016}.  \end{proof}

Theorem \ref{thm:unique} is based on the necessary conditions presented in Theorem 3 of \cite{Huang:2013}.  The proof presented in \cite{Adams:2016} adjusts the proof of \cite{Huang:2013} to account for the extra constraint provided by Assumption \ref{asm:addup}.  The proof shows that the necessary condition is the same.  This means that the minimum requirement for identification is the sparsity condition presented in \cite{Huang:2013}.  The paper shows that the additional constraint of adding up implies that this necessary condition is also sufficient.

\section{Estimation}

The previous section showed that it is possible to uniquely determine the parameters of the stochastic matrix factorization problem.  This section shows that it is possible to consistently estimate the matrix factor whose dimensions do not increase as the data set increases.

\subsection{Picture Reduction Problem}

For a gray-scale picture, pixel $j$ in picture $i$ is given as 
\begin{equation}
x_{ij} = \sum_{r = 1}^{R} w_{ir} h_{rj} + e_{ij}
\end{equation}
where $h_{rj} \in [0,1]$, $w_{ir} \in [0,1]$, $\sum_{r=1}^R w_{ir} = 1$, $e_{ij} \in \Re$, and $E(e_{ij}) = 0$ for all $i \in \{1,...,N\}$, $j \in \{1,...,M\}$ and $r \in \{1,...,R\}$.  In matrix notation we have
\begin{equation}
\mathbf{X} = \mathbf{W} \mathbf{H} + \mathbf{E}
\end{equation}
where in expectation $\mathbf{E} = 0$.  

If this model seems familiar, it is because it is a linear factor model \citep{Bai:2003}.  Each picture has a common component determined by the weighting vector and there are a set of ``factors'' or ``base images'' that are true for the whole set of pictures.  The factorization is an approximation of the actual picture, where that approximation is assumed to be unbiased at the pixel level.  The difference between this problem and a more standard linear factor model is that the values of the pixels naturally lie between 0 and 1 (assuming a gray-scale image analyzed here).  This leads naturally to the restriction that the pixels of the base images also lie between 0 and 1.

If we have a random sample of $N_s$ pictures we can factorize the sample such that
\begin{equation}
\mathbf{X}_s = \mathbf{W}_s \mathbf{H}_s
\end{equation}
where $s$ denotes the fact that the data matrix and the matrix factors are from the sample.
This can be rearranged to give 
\begin{equation}
\mathbf{X}_s \mathbf{H}_s^{+} = \mathbf{W_s}
\end{equation}
where $\mathbf{H}_s \mathbf{H}^{+} = \mathbf{I}$ denotes the generalized inverse of one of the matrix factors.  This inverse exists if $\mathbf{H}$ has rank of $R$.  

We can write the estimator as follows, where $||\cdot||_F$ denotes the Frobenius norm.  This is a least squares estimator where the objective is to minimize the squared distances between the cells.

\begin{equation}
\begin{array}{ll}
\min_{\{\mathbf{H}_s\}}  & ||\mathbf{X}_s - \mathbf{X}_s \mathbf{H}_s^{+} \mathbf{H}_s)||_F \\
    & \\
s.t. &  \mathbf{X}_s \mathbf{H}_s^{+} \mathbf{1}_R = \mathbf{1}_R\\
      & \mathbf{X}_s \mathbf{H}_s^{+} \ge 0\\
      &  \mathbf{H}_s \ge 0\\
      & \mathbf{H}_s \le 1
\end{array}
\label{eq:estimator}
\end{equation}

At its heart this is just a least-squares optimizer.  One issue is that the adding up and non-negativity constraints provide some additional ``moments'' on the estimator.  In the paper the estimator constrains the parameters to be between zero and one in a standard way.  The additional adding up  and non-negativity constraints are placed in the optimizer as an additional conditions to be minimized.  These additional conditions are weighted (see discussion in the results section).\footnote{See \cite{Hansen:1982} for a discussion of optimal weights for such an estimator.  Here the weights are chosen somewhat arbitrarily (``tuned'').}

Given that we have a method for determining $H_s$, is there any sense in which this is the matrix we are looking for?  The following corollary states that the estimator is consistent.  That is, if we were able to get a ``large enough'' data set, then the matrix estimated is equal to the matrix of interest.

\begin{cor}  Let
\begin{equation}
\mathbf{x}^{T} = \mathbf{w}^{T} \mathbf{H}
\end{equation}
where $\mathbf{x}$ is a $M \times 1$ column matrix that represents the ``average picture'' and $\mathbf{w}$ is a $R \times 1$ column matrix which represents the distribution over base images in the population of images.  Let $M > R$ and the assumptions of Theorem \ref{thm:unique} hold for the observed data then as $N_s \rightarrow \infty$, $\mathbf{H}_s = \mathbf{H}$.
\label{cor:large}
\end{cor}

\begin{proof}
By Theorem \ref{thm:unique} and random sampling, averaging across pictures we have
\begin{equation}
\frac{1}{N_s} \mathbf{1}_{N_s}^T \mathbf{X}_s = \frac{1}{N_s} \mathbf{1}_{N_s}^T \mathbf{W}_s  \mathbf{H}_s +  \frac{1}{N_s} \mathbf{1}_{N_s}^T \mathbf{E}_s 
\end{equation}
As $N_s \rightarrow \infty$, 
\begin{equation}
\mathbf{x}^T = \mathbf{w}^T  \mathbf{H}_s  
\end{equation}
\end{proof}

This suggests that we can find a consistent estimator as the number of pictures gets large.

\subsection{Topic Models}  

A topic model is a standard method for representing and analyzing large document sets.  In this model, a document is assumed to be represented by a ``bag of words''.  That is, the word placement is ignored and only the word count is used.  The probability of observing term $j$ in document $i$ is 
\begin{equation}
\Pr(word_{ij} = term) = \sum_{r = 1}^R \Pr(topic_r | document_i) \Pr(word_j = term | topic_r)
\end{equation}
or
\begin{equation}
x_{ij} = \sum_{r=1}^R w_{ir} h_{rj}
\end{equation}
where $x_{ij}$ is the probability of term $j$ occurring in document $i$, $w_{ir}$ is the probability of topic $r$ occurring given document $i$ and $h_{rj}$ is the probability of term $j$ occurring given topic $r$.  In this model, the distribution of terms in the document is a mixture model where the mixture probabilities vary across the document set.  

In matrix notation we have
\begin{equation}
\mathbf{X} = \mathbf{W} \mathbf{H}
\end{equation}
Note that the matrix $\mathbf{X}$ is a general representation of the document corpus given the bag of words assumption.  The other assumption of the model is that the correlation of terms in the document is determined by the conditional mixture distribution.

If we have a random sample of $N_s$ documents we have our sample version of the matrix factorization problem.
\begin{equation}
\mathbf{X}_s = \mathbf{W}_s \mathbf{H}_s
\end{equation}
This model can be estimated with an estimator that is almost the same as the one used for the picture reduction problem.  The difference is that there is one additional constraint.  Here the rows of the $\mathbf{H}_s$ matrix are assumed to add to one.  This is a straightforward constraint to add to the estimator.\footnote{Note that it is not necessary to include the additional constraints on the $\mathbf{W}_s$ matrix, and removing them may speed up the estimator.}  The consistency result is also similar but not quite the same.

\begin{cor}  Let
\begin{equation}
\mathbf{x}^{T} = \mathbf{w}^{T} \mathbf{H}
\end{equation}
where $\mathbf{x}$ is a $M \times 1$ column matrix that represents the word distribution across the corpus and $\mathbf{w}$ is a $R \times 1$ column matrix which represents the distribution over topics across the corpus.  Let $M > R$ and the assumptions of Theorem \ref{thm:unique} hold for the observed data then as $N_s \rightarrow \infty$, $\Pr(\mathbf{H}_s - \mathbf{H} = 0) = 1$.
\label{cor:large2}
\end{cor}

\begin{proof}
By Theorem \ref{thm:unique},
$\mathbf{H}_s$ is uniquely defined (up to relabeling) by 
\begin{equation}
\mathbf{X}_s = \mathbf{W}_s \mathbf{H}_s
\end{equation}
Averaging across documents we have
\begin{equation}
\frac{1}{N_s} \mathbf{1}_{N_s}^T \mathbf{X}_s = \frac{1}{N_s} \mathbf{1}_{N_s}^T \mathbf{W}_s  \mathbf{H}_s 
\end{equation}
Let $\mathbf{H}_s = \mathbf{H} + \mathbf{E}_s$ where $\mathbf{H}$ is as defined above.  Substituting in, we have
\begin{equation}
\frac{1}{N_s} \mathbf{1}_{N_s}^T  \mathbf{X}_s = \frac{1}{N_s} \mathbf{1}_{N_s}^T \mathbf{W}_s \mathbf{H}\ + \frac{1}{N_s} \mathbf{1}_{N_s}^T \mathbf{W}_s \mathbf{E}_s 
\end{equation}
By the law of large numbers $\frac{1}{N_s} \mathbf{1}_{N_s}^T \mathbf{X}_s \rightarrow \mathbf{x}^T$ and  $\frac{1}{N_s} \mathbf{1}_{N_s}^T \mathbf{W}_s  \rightarrow \mathbf{w}^{T}$, so for large enough $N_s$
\begin{equation}
\mathbf{x}^{T} =\mathbf{w}^{T} \mathbf{H}  +  \mathbf{w}^{T} \mathbf{E}_s  
\end{equation}
By assumption, $\ \mathbf{w}^{T} \mathbf{E}_s \rightarrow \mathbf{0}_M$, where $\mathbf{0}_M$ is a $M \times 1$ column matrix of zeros.  From this, each element of $\mathbf{E}_s$ must be zero with probability one.
\end{proof}

Corollary \ref{cor:large2} states that if the identification conditions hold in the observed data, then as the number of documents in the sample gets large, the estimated parameters converge to the true values with probability one.

\section{The Face Data}

The paper illustrates the stochastic matrix factorization method using a data set analyzed by \cite{Lee:1999} in their original paper on non-negative matrix factorization.  The data set includes a little over 2,400 pictures of faces.  The original pictures were compressed to $19 \times 19$ gray-scale images.  Here, the pictures have been simplified to $9 \times 9$ gray-scale images by averaging over $2 \times 2$ pixel blocks and removing the bottom and far right edges.  In the original paper, the authors estimate a model with 49 base images.  Here we will assume 10 base images.

Given the aggregation, the smaller number of base images and the method described above we have 810 parameters to estimate using a little under 2,500 observations.  This compares to the original paper which estimated some 137,000 parameters with on the same data.  I note that the number of parameters is only 18,000 if the above methodology is used.  The approximation is estimated using a least squares estimator with additional moments to account for the constraint that one of the matrices is stochastic.  In the estimator used, the weighting on summation to 1 requirement is set to 100, and the weighting on the non-negativity requirement is set to 10.  Note that these weights are relative to 1 for each cell difference being minimized.\footnote{A number of different weighting schemes were experimented with.}

\begin{figure}[htp]
\includegraphics[scale=0.5]{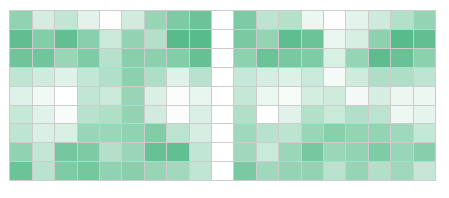}\\
\includegraphics[scale=0.1]{"face1"}
\caption{Left: Original, Right: Approximation, Bottom: 1/5th scale\label{fig:face1}}
\end{figure}

Figure \ref{fig:face1} presents the original data and the approximation based on the matrix factorization for the first image in the data file.  This illustrates that the approximation reduces the amount of information that exists for the picture, however it still seems to be a relatively good approximation.  

\begin{figure}[htp]
\includegraphics[scale=0.5]{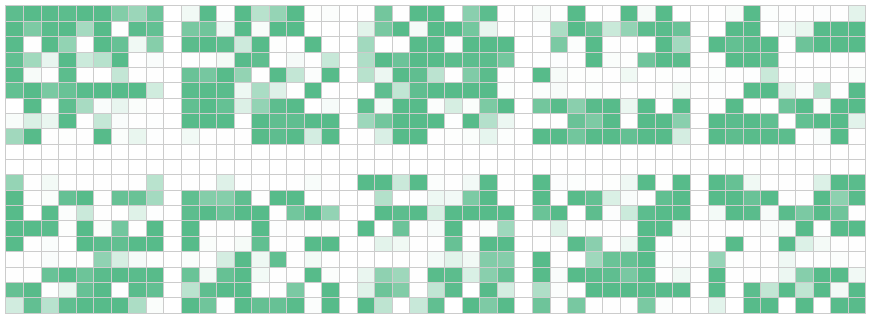}
\caption{Ten basis images\label{fig:face_basis}}
\end{figure}

Figure \ref{fig:face_basis} presents the ten basis images or the $\mathbf{H}$ matrix of the factorization.  In the original paper, the authors estimate 49 basis images and seem to find basis that represents ``parts'' of the face, i.e. the eyes, noes, mouth etc.  Here there is no discernible ``parts,'' but neither does a `` basis'' represent some distorted or averaged ``face.''

The matrix factorization method is able to reduce the number of pixels to be saved to approximately 12\% of the number in the original file ($9 \times 9$ images).  In addition, the files can be retrieved with an error rate of 0.0012.\footnote{Note that this error rate is not improved when the original $9 \times 9$ images are used for retrieval.}  In addition, this method reduces the retrieval time by about half compared to using the full images (22 seconds to retrieve each of the images in the data set, versus 48 seconds.)

As mentioned above, this procedure does not generally provide a unique solution to the matrix factorization problem.  Rather the model is ``set identified''.  That is, depending on the data, the parameters are estimated up to intervals on the real line.  We can get some idea of how large these bounds are by looking at the identified set along the axes.  Here we have 180 points to check, the two end points for each of the 90 axes.  Using this procedure we can determine that most intervals are less than one percentage point wide, but there are 5 of the 90 that are between 1 and 2 percentage points.  No intervals are greater than 2 percentage points.  Intuitively, this could mean an over or underweighting of a basis image by up to 2 percentage points.

\section{Econ PhD Dissertation Abstracts}

The paper uses a data set from PhD dissertation abstracts in economics to illustrate how SMF can be used to estimate a non-parametric topic model.  The abstracts are based on the Proquest dissertation database and cover the years 2002 to 2014.  The abstracts included all abstracts where one of the three subjects of the abstract contained the term ``economics''.\footnote{Only some ten percent of entries in the data base contain information about the home department.  I am grateful to Matthew Chesnes for his help with the data set.}  Table \ref{tab:subjects} presents the distribution of subject terms.  Note that over half the dissertations have no entry for a third term and about a third have the entry for the first term of ``economics, general.''

\begin{table}[htdp]
\begin{center}
\begin{tabular}{|| l |c | c | c ||}
\hline \hline
	&	Term 1	&	Term 2	&	Term 3	\\ \hline
Economics, general	&	8,587	&	1,946	&	647	\\
Economics, finance	&	3,587	&	2,122	&	437	\\
Economics, agricultural	&	1,332	&	589	&	264	\\
Economics, theory	&	1,300	&	1,523	&	505	\\
Economics, labor	&	1,299	&	997	&	343	\\
Business Administration, management	&	601	&	435	&	206	\\
Economics, commerce-business	&	464	&	835	&	372	\\
Other	&	5,529	&	6,753	&	4,360	\\
No Value	&	0	&	7,449	&	15,515	\\
\hline \hline
\end{tabular}
\vspace{.05in}
\caption{Distribution of the values in the three subject fields.  \label{tab:subjects}}
\end{center}
\end{table}

Table \ref{tab:categories} gives a bit more of a sense of the dissertation data.  We see that the old joke about dissertation titles ``three essays in ...'' is based in fact.  UC Berkeley leads all schools in terms of the number of dissertations granted.  Reassuringly, economics departments dominate, although only a small number of the abstracts contain information about the department.  Claremont's Thomas Willett leads all comers in the number of dissertations advised, with Berkeley labor economist, David Card and Minnesota macroeconomist, Tim Keheo right behind.

\begin{table}[htdp]
\begin{center}
\begin{tabular}{|| l |c ||}
\hline \hline
Title	&	Count	\\ \hline
Essays in Financial Economics	&	71	\\
Essays in Applied Microeconomics	&	45	\\
Essays in Corporate Finance	&	40	\\
Essays in International Economics	&	23	\\
Essays in Macroeconomics	&	22	\\
Essays in Labor Economics	&	21	\\ \hline \hline
School	&		\\ \hline
UC Berkeley	&	611	\\
Harvard	&	551	\\
Chicago	&	524	\\
Maryland	&	429	\\
Columbia	&	426	\\
Minnesota	&	412	\\ \hline \hline
Department	&		\\ \hline
Economics	&	2,955	\\
Business  	&	191	\\
Business Administration 	&	177	\\
Finance	&	173	\\
Agricultural Economics	&	123	\\
Other	&	2,801	\\
No Value	&	16,229	\\ \hline \hline
Advisor	&		\\ \hline
Thomas Willett	&	43	\\
David Card	&	37	\\
Tim Kehoe	&	32	\\
Andrei Shleifer	&	31	\\
Gary Becker	&	26	\\
Other	&	18,852	\\
No Value	&	3,628	\\
\hline \hline
\end{tabular}
\vspace{.05in}
\caption{Counts of dissertations for various categories. \label{tab:categories}}
\end{center}
\end{table}

Using this data, the topic model presented in Equation (\ref{eq:estimator}) is estimated assuming there are twenty topics and three hundred and sixty terms in the term-dictionary.  The term-document matrix is constructed by removing ``sparse'' terms, ``stop words," numerals, and stemming the remaining words.  The document count is reduced to include only documents with a positive number of the remaining terms.  The matrix is then normalized to create a distribution over the term-dictionary for each document.  The estimation problem is to factorize a $360 \times 22,649$ matrix into a $360 \times 20$ matrix and a $20 \times 22,649$ matrix.  The procedure presented above is used to estimate the approximately 7,000 parameters that represent the matrix of ``basis'' documents.

\begin{table}[htdp]
\begin{center}
\begin{tabular}{|| c |c |c|c|c|c||}
\hline \hline
Topic	&	1	&	2	&	3	&	4	&	5	\\ \hline
1	&	lead	&	program	&	becaus	&	decad	&	motiv	\\
2	&	often	&	final	&	dissert	&	public	&	problem	\\
3	&	interact	&	state	&	result	&	benefit	&	present	\\
4	&	found	&	observ	&	issu	&	use	&	reduc	\\
5	&	explain	&	key	&	control	&	way	&	even	\\
6	&	key	&	becom	&	indic	&	agricultur	&	impact	\\
7	&	differ	&	analyz	&	need	&	improv	&	credit	\\
8	&	bias	&	investig	&	thesi	&	dynam	&	negat	\\
9	&	need	&	argu	&	exchang	&	network	&	manag	\\
10	&	lead	&	return	&	time	&	investor	&	assess	\\
11	&	requir	&	identifi	&	test	&	small	&	factor	\\
12	&	find	&	incom	&	util	&	stock	&	achiev	\\
13	&	purpos	&	strong	&	requir	&	show	&	investor	\\
14	&	toward	&	higher	&	well	&	onli	&	greater	\\
15	&	labor	&	outcom	&	control	&	network	&	correl	\\
16	&	two	&	human	&	interact	&	distribut	&	affect	\\
17	&	work	&	structur	&	probabl	&	account	&	flow	\\
18	&	low	&	regard	&	base	&	point	&	approach	\\
19	&	toward	&	play	&	bias	&	larger	&	present	\\
20	&	agent	&	behavior	&	structur	&	assess	&	popul	\\
\hline \hline
\end{tabular}
\vspace{.05in}
\caption{Top five terms by frequency distribution conditional on the topic (rows). \label{tab:term_dist}}
\end{center}
\end{table}

Table \ref{tab:term_dist} presents the top five terms by topic, where the topics are the rows of the table.  These term-distributions tend to be highly skewed and, while not reported, the ``anchor word'' condition seems to hold in the data.

\begin{figure}[htdp]
\includegraphics[scale=0.5]{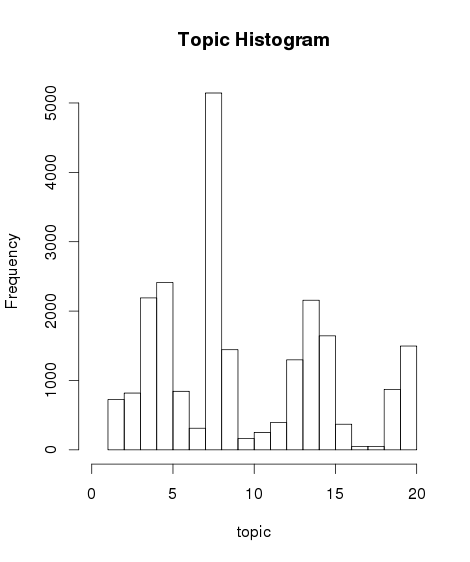}
\caption{Distribution of ``most-probable'' topics across the document corpus. \label{fig:topic_hist}}
\end{figure}

As stated above, the distribution of topics conditional on the document is calculated post-estimation using Bayes' rule, the document's word distribution and the parameter values from the estimation.  Again, it is not reported but this matrix also seems to satisfy the ``anchor document'' condition for a large number of documents in the corpus.  Figure \ref{fig:topic_hist} present a histogram of the distribution of the ``most probable'' topics for each document across the corpus.  This figure shows that there is large amount of variation across topics.

While, in general, this model is ``set-identified'' the bounds on the estimated parameters seem to be very tight.  Using the example presented above we can look at the likelihood ratios to find the intervals along each of the 380 axes representing the identified set.  Almost all intervals are smaller than 1/10 of one percentage point, with a few that are higher, but none is greater than one percentage point.  The necessary and sufficient sparsity conditions presented above seem to hold in the data set analyzed here.  Note that the paper does not estimate or account for sampling variance.

\subsection{Prediction Problem}

The problem analyzed here is to predict the hand entered subject term.  To analyze the value of the topic model this subsection uses the results of the topic model above to estimate a multinomial logit model of the dissertation's first subject term.  In the analysis the subject term is assumed to be either ``economics, finance,'' ``economics, labor,'' ``economics, theory,'' ``economics, agricultural'' and an ``other category.''  The results of the topic model are included by taking the ``most probable'' topic for each document and including that as a factor in the model.  The model also includes continuous measures of time (year) and the number of pages that the dissertation has.

\begin{table}[htdp]
\begin{center}
\begin{tabular}{||l | c |c ||c|c||c|c||c|c||}
\hline \hline
 & \multicolumn{2}{c ||}{Finance} & \multicolumn{2}{c ||}{Labor}
& \multicolumn{2}{c ||}{Theory} & \multicolumn{2}{c ||}{Ag}\\ \hline \hline
             &  Coef  & SE        &  Coef   & SE 	    &  Coef    & SE      &  Coef & SE \\ \hline
Constant	&	3.78	&	0.00	&	1.07	&	0.00	&	5.62	&	0.00	&	7.69	&	0.00	\\
Topic 2	&	0.08	&	0.18	&	1.03	&	0.28	&	-0.51	&	0.32	&	-0.34	&	0.09	\\
Topic 3	&	-0.47	&	0.18	&	0.59	&	0.27	&	-0.29	&	0.23	&	0.12	&	0.25	\\
Topic 4	&	-0.06	&	0.13	&	0.72	&	0.21	&	-0.67	&	0.23	&	0.01	&	0.22	\\
Topic 5	&	-0.29	&	0.08	&	0.78	&	0.12	&	-0.62	&	0.13	&	0.07	&	0.12	\\
Topic 6	&	-0.10	&	0.06	&	0.65	&	0.10	&	-0.59	&	0.10	&	0.11	&	0.10	\\
Topic 7	&	-0.18	&	0.07	&	1.14	&	0.10	&	-0.56	&	0.11	&	0.10	&	0.11	\\
Topic 8	&	-0.73	&	0.16	&	0.78	&	0.21	&	-0.31	&	0.19	&	-0.17	&	0.23	\\
Topic 9	&	-1.37	&	0.00	&	0.94	&	0.01	&	-1.01	&	0.00	&	-0.44	&	0.00	\\
Topic 10	&	-0.44	&	0.01	&	0.32	&	0.00	&	-0.21	&	0.01	&	-7.06	&	0.00	\\
Topic 11	&	-0.27	&	0.09	&	0.46	&	0.16	&	-0.61	&	0.15	&	-0.05	&	0.15	\\
Topic 12	&	-0.16	&	0.10	&	0.30	&	0.20	&	-0.34	&	0.15	&	-0.20	&	0.18	\\
Topic 13	&	-0.25	&	0.07	&	0.63	&	0.12	&	-0.44	&	0.11	&	-0.13	&	0.12	\\
Topic 14	&	-0.40	&	0.10	&	0.70	&	0.15	&	-0.45	&	0.14	&	-0.06	&	0.15	\\
Topic 15	&	-0.49	&	0.07	&	0.73	&	0.10	&	-0.64	&	0.10	&	0.21	&	0.09	\\
Topic 16	&	-0.55	&	0.06	&	0.65	&	0.09	&	-0.50	&	0.09	&	0.20	&	0.09	\\
Topic 17	&	-0.53	&	0.10	&	0.61	&	0.15	&	-0.83	&	0.16	&	-0.02	&	0.15	\\
Topic 18	&	-0.01	&	0.14	&	0.31	&	0.28	&	-0.31	&	0.22	&	-0.45	&	0.29	\\
Topic 19	&	-0.32	&	0.05	&	0.65	&	0.07	&	-0.55	&	0.07	&	-0.01	&	0.07	\\
Topic 20	&	-0.27	&	0.07	&	0.66	&	0.12	&	-0.57	&	0.12	&	0.24	&	0.11	\\
Degree Year	&	Yes	&		&	Yes	&		&	Yes	&		&	Yes	&		\\
Pages	&	Yes	&		&	Yes	&		&	Yes	&		&	Yes	&		\\ \hline
RMSE & \multicolumn{2}{c||}{1.2917}  & & & & & & \\
\hline \hline
\end{tabular}
\vspace{.05in}
\caption{Multinomial regression results.  Note that the topic numbers do not correspond to the numbering used in Table \ref{tab:term_dist} (unfortunately). \label{tab:mn_res}}
\end{center}
\end{table}

Table \ref{tab:mn_res} presents the results from the multinomial logit.  For each choice, the left-hand column presents the coefficient estimate and the right-hand columns gives the standard error.   The last two rows indicate that continuous variables for year of publication and the number of pages were included in the regression model.  The results suggest that Topic 7 is an important indicator of a labor dissertation, while Topic 10 dissertations are very unlikely to be Ag.

As a measure of fit, the paper uses the root mean squared error.  After the model is estimated, the fitted values are calculated for the data and the squared difference between the observed subject and the predicted subject.  The root of the average squared difference is then calculated.  The RSME for the model using the results from above are 1.2917.  This compares to 1.2927 for a more standard topic model using LDA to calculate the most probable topics.

\section{Conclusion}

This paper considers a ``natural'' restriction on the matrix factorization problem where both matrix factors are assumed to be non-negative and at least one of the matrix factors is assumed to be stochastic (its rows or columns add to 1).  This stochastic matrix factorization is a restriction on the non-negative matrix factorization model \citep{Lee:1999}.  However, for many applications of non-negative matrix factorization, the restriction may be of no importance.  In particular, the two problems analyzed in \cite{Lee:1999} can be modeled as stochastic matrix factorization problems.  The paper shows that stochastic matrix factorization problems are unique when the data satistifies the conditions posited in \cite{Huang:2013}.  These conditions are necessary and sufficient for uniqueness of the factorization.  The paper illustrates the model using the ``face'' date analyzed by \cite{Lee:1999} and a topic model analysis of dissertation abstracts in economics.

\bibliographystyle{plainnat}
\bibliography{ei}

\end{document}